\documentclass[letterpaper]{article}
\usepackage{uai2020}
\usepackage[margin=1in]{geometry}
\usepackage{amsfonts}
\usepackage{amsmath}

\usepackage{times}

\newcommand{\pp}[1]{\ensuremath{\mathbb P\left\{#1\right\}}}
\newcommand{\E}[1]{\ensuremath{\mathbb E\left\{#1\right\}}}
\newcommand{\EE}[2]{\ensuremath{\mathbb E_{#1}\left\{#2\right\}}}
\newcommand{\td}{\tilde}
\newcommand{\ve}{\mathbf}
\newcommand{\Tr}{\textup{Tr}}
\newtheorem{definition}{Definition}
\newtheorem{condition}{Condition}
\newtheorem{lemma}{Lemma}
\newtheorem{theorem}{Theorem}
\newtheorem{remark}{Remark}
\newenvironment{proof}{\paragraph{Proof:}}{\hfill$\square$}

\title{Semi-Supervised Learning: the Case When Unlabeled Data is Equally Useful}

\author{ {\bf Jingge Zhu} \\
Department of Electrical and Electronic Engineering \\
University of Melbourne\\
jingge.zhu@unimelb.edu.au
}

\begin{document}

\maketitle

\begin{abstract}
Semi-supervised learning algorithms attempt to take advantage of relatively inexpensive unlabeled data to improve learning performance. In this work, we consider statistical models where the  data distributions  can be characterized by  continuous parameters. We show that under certain conditions on the distribution, unlabeled data is equally useful as labeled date in terms of learning rate. Specifically, let $n, m$ be the number of labeled and unlabeled data, respectively. It is shown that the learning rate of semi-supervised learning scales as $O(1/n)$ if $m\sim n$, and scales as $O(1/n^{1+\gamma})$ if $m\sim n^{1+\gamma}$ for some $\gamma>0$, whereas  the learning rate of supervised learning scales as $O(1/n)$. 
\end{abstract}

\section{INTRODUCTION}

It is known that in favorable situations, semi-supervised learning (SSL) is able to take advantage of unlabeled data to improve learning performance. In this work,  we  study how learning rate (defined to be the convergence rate of the excess risk in this paper) is improved by having additional unlabeled data under a parametrization assumption of the data distribution. Our main finding is that under our assumption and certain conditions on the data-generating distribution, unlabeled data is as useful as the labeled data in terms of learning rate. 

Numerous works across the past few decades are devoted to understand the role of unlabeled data in learning problems. The early work of \cite{castelli_relative_1996} studied a simple mixture model and showed the relative value of labeled and unlabeled data under different assumptions of the model. The author in \cite{rigollet_generalization_2007} formally formulated the notion of \textit{cluster assumption} and proposed a method that takes advantage of unlabeled data to achieve fast convergence rates. A more sophisticated mixutre model was studied in \cite{singh_unlabeled_2009}, where  different regimes of parameters are identified in which the unlabeled data help.  The recent work \cite{gopfert_when_2019} gave an overview of various assumptions in different works. The readers are referred to \cite{chapelle_semi-supervised_2006} \cite{zhu_review} for a   comprehensive literature review on the topic. This paper uses a different assumption than most previous works, and we will comment on their differences in Section \ref{sec:compare} after presenting our main results.

In this work, we view  both supervised and semi-supervised learning problem as a variation of the universal prediction problem~\cite{merhav_universal_1998}. In the classical setup of universal prediction, an observer sequentially receives a sequence of observations $x_1,x_2,\ldots$, and wishes to predict the  next outcome $x_t$  based on all past observations up to  time $t-1$. The exact underlying distribution that generates the data is  unknown to the predictor,  except that it comes from a  family of  parametrized distributions. The goal is to design a \textit{universal predictor} that performs well  in the absence of the exact knowledge of the distribution.  The connection to the learning problem is that instead of considering a sequential prediction problem, we assume that all past observations (i. e. training data) are given, and only one prediction needs to be made (for the testing data).  Importantly, we still assume that the data-generating distribution is not exactly known except that it comes from a parameterized family.

The main contributions of this paper are summarized as follows.
\begin{itemize}
\item For  some widely used loss functions, we provide an upper bound on the excess risk (Lemma \ref{lemma:expconcave}) characterized by a conditional mutual information term. This bound could be interesting on its own.
\item Using the above upper bound, we obtain the learning rate of supervised and semi-supervised learning problems (Theorem \ref{thm:main}).  Let $n, m$ be the number of labeled and unlabeled data, respectively. We show that under certain conditions (to be specified in  Section \ref{sec:main}), the rate of semi-supervised learning scales as $O(1/n)$ if $m\sim n$, and scales as $O(1/n^{1+\gamma})$ if $m\sim n^{1+\gamma}$ for $\gamma>0$, whereas  the learning rate of supervised learning scales as $O(1/n)$.  We also identify the corresponding constant in the leading term in each case. This shows that under appropriate conditions, the unlabeled data is equally useful as the labeled data insofar as the convergence rate is concerned.
\item A lower bound on the learning rate of supervised learning algorithms with a certain type of loss function is given (Lemma \ref{lemma:lower_SL}), showing that our characterization of the learning rate is tight.
\end{itemize}

\section{PROBLEM STATEMENT}
\label{sec:problem}

Let $(X,Y)$ be a pair of random variables with the density function $p_{\theta}(x,y)$ where $\theta\in\Lambda$, and the set $\Lambda$ is  a measurable set in $\mathbb R^d$.  We assume $X\in\mathcal X$ and $Y\in\mathcal Y$ where $\mathcal X$ is an arbitrary feature space, and $\mathcal Y$ is a discrete set consisting of  labels. With a slight abuse of  notations, we use $p_{\theta}(x):=\sum_{y} p_{\theta}(x,y)$ and $p_{\theta}(y):=\int_x p_{\theta}(x,y)dx$ to denote the marginal distributions of $X$ and $Y$, respectively. The distribution $p_{\theta}(x,y)$ can also be seen as the conditional distribution of $(X,Y)$ given the parameter value $\theta$. Hence throughout the paper, we will use the notations $p_{\theta}(\cdot)$ and $p(\cdot|\theta)$ interchangeably.  We point out that our results can be extended straightforwardly to the case when $Y$ is a continuous random variable.

 Let $w: \mathcal X\rightarrow \mathcal D$ be a hypothesis (classifier/predictor) that maps each element in $\mathcal X$ to an  element in the space $\mathcal D$.  It is most natural to take $\mathcal D$  to be $\mathcal Y$, where the mapping $w$  returns a label for each $x$. However, we also allow $\mathcal D$ to be different from $\mathcal Y$. For example, we could take $\mathcal D$ to be the probability simplex of dimension $|\mathcal Y|$, where each element in $\mathcal D$ is a nonnegative vector summing up to unity. In this case, the mapping $w$ returns a \textit{probability assignment} on $y$ for each $x$.

Given the hypothesis $w$ and a pair $(x,y)$, the \textit{risk} is defined as $\ell(w(x), y)$   for some loss function $\ell: \mathcal D\times \mathcal Y \rightarrow \mathbb R$. To lighten  notations, we often use $Z$ to represent a pair $(X,Y)$, and write $\ell(w(X),Y)$ simply as $\ell(w,Z)$.

For a given hypothesis $w$, the expected risk is defined as
\begin{align}
L_{\theta}(w):=\EE{\theta}{\ell(w,Z)},
\end{align}
where the subscript denotes that the expectation is taken with respect to $Z\sim p_{\theta}(Z)$. We define  $w^*$ to be the \textit{Bayes hypothesis} that minimizes of the expected risk
\begin{align*}
w^*:=\text{argmin}_w \EE{\theta}{\ell(w,Z)}.
\end{align*}
The  \textit{excess risk} of a given hypothesis $w$ is  defined to be
\begin{align*}
R_{\theta}(w)&:=L_{\theta}(w)-L_{\theta}(w^*).
\end{align*}
Notice that in general $w^*$ depends on the  distribution $p_{\theta}$, whereas $w$  only has access to a finite number of samples.



We consider two different learning scenarios.
\begin{itemize}
\item 1) \textbf{Supervised learning.}  Let  $w_{Z^n}$ denote the hypothesis generated by  labeled data $Z^n=(Z_1,\ldots, Z_n)$.  Assume $Z_i$ are i.i.d data distributed according to $p_{\theta_0}(z)$ where $\theta_0$ is the true parameter. The  optimal expected excess risk  of supervised learning is  
\begin{align*}
R_{SL}(\theta_0):=\min_w\E{R_{\theta_0}(w_{Z^n})},
\end{align*}
where the  expectation  is taken with respect to the labelled data $Z^n$.
\item 2) \textbf{Semi-supervised learning (SSL).} Let  $w_{Z^n,\td X^m}$ denote the  hypothesis generated with the labeled data $Z^n=(Z_1,\ldots, Z_n)$ and additional unlabeled data $\td X^m=(\td X_1,\ldots, \td X_m)$.  Also assume that $\td X_i$ are i.i.d data distributed according to $p_{\theta_0}(x)$ with $\theta_0$ being the true parameter. The  optimal expected excess risk  of supervised learning is  
\begin{align*}
R_{SSL}(\theta_0):=\min_w \E{R_{\theta_0}(w_{Z^n,\td X^m})},
\end{align*}
where the  expectation  is taken with respect to the labeled data $Z^n$ and unlabeled data $\td X^m$.
\end{itemize}

Throughout the paper, we assume that the density function $p_{\theta_0}(x,y)$ does \textit{not} depend on the number of samples $n$ an $m$.

In this work,  we state all our results for some given $\theta_0$.  We point out that it is also possible to work with a minimax setup by defining the minimax excess risk as $R_{SL}:=\min_w\max_{\theta }\E{R_{\theta}(w_{Z^n})}$ and so on.  Similar results can be derived within the minimax problem formulations.

\section{UPPER BOUNDS ON RISKS}
\label{sec:bounds}
In the following, we  give upper bounds on $R_{SL}$ and $R_{SSL}$ in terms of  (conditional) mutual information involving $Z^n, \td X^m$ and an auxiliary random variable $\Theta$ defined over $\Lambda$. Recall that the (conditional) mutual information $I(X;Y|Z)$ is defined as
\begin{align*}
I(X;Y|Z)&:=\int p_{XYZ}(x,y,z)\log \frac{p_{Y|XZ}(y|x,z)}{p_{Y|Z}(y|z)}dxdydz.
\end{align*}
We will use the notation $I(X=x; Y|Z)$ to denote mutual information conditioned on $X=x$
\begin{align}
&I(X=x; Y|Z):=\nonumber\\
&\int p_{YZ|X}(y,z|x)\log \frac{p_{Y|XZ}(y|x,z)}{p_{Y|Z}(y|z)}dydz,
\label{eq:mutual_info}
\end{align}
which can also be written as
\begin{align}
I(X=x; Y|Z)=D(p_{Y|X=x, Z}||p_{Y|Z}|p_{Z|X=x}),
\label{eq:mutualinfo_cond}
\end{align}
where $D(p_{X|Y}||q_{X|Y}|r_Y)$ denotes the conditional Kullback-Leibler (KL) divergence
\begin{align*}
D(p_{X|Y}||q_{X|Y}|r_Y):=\int r_Y(y) D(p_{X|Y=y}||q_{X|Y=y})dy.
\end{align*}

The upper bounds to be derived are inspired by the classical universal prediction problem, where  a non-negative quantity called \textit{minimax redundancy} plays an important role. It is the smallest possible worst-case difference between the risk incurred by a universal predictor and that incurred by a predictor that \textit{knowns} the true distribution of the data. It is  well known that the minimax redundancy is equivalent to the maximin redundancy under some assumptions on the loss function \cite{Gallager_capacity}, which can be characterized as the capacity of a ``channel" (hence in the form of  mutual information), where the input is the parameter that characterizes the distribution, and the output is the generated data. The next two lemmas  could also be interesting on their own due to their connection to the information-theoretic quantity.


We first give an upper bound when the loss function belongs to the class of exponentially concave functions, defined as follows.
\begin{definition}[Exponentially concave function]
A function $f:\mathcal D\rightarrow \mathbb R$ is called a  $\beta$-exponentially concave function in $x\in\mathcal D$ for some $\beta>0$ if $\exp(-\beta f(x))$ is concave.
\label{def:expconcave}
\end{definition}

The class of exponentially concave (exp-concave) functions  are widely used as loss functions in machine learning problems. For example,  it is  easy to verify that the square loss $(b-x)^2$ is  $1/(8a^2)$-exp-concave  if the absolute value of $b,x$ are no larger than $a$. It is also  shown in \cite{alirezaei_exponentially_2018} that both discrete entropy and Renyi entropy, when appropriately scaled, are exp-concave functions. Another important $1$-exp-concave function is the so-called  \textit{self-information} loss function \cite{merhav_universal_1998}, is defined as
$\ell(w(x), y)=-\log w(y)$
where $w$ is a probability assignment (depending on $x$) of $y$. In other words, $w$ can be thought as a length-$|\mathcal Y|$ nonnegative vector summing up to $1$, and $w(y)$ returns the value of the entry corresponding to $y$. Furthermore, the cross entropy loss function can also be shown to be exponentially concave.

\begin{lemma}[Upper bound on  risk for exp-concave loss]
Assume that $\ell(w, z)$ is a $\beta$-exponentially concave function of $w$ for all $z$.  Then for any true parameter $\theta_0\in\Lambda$, it holds that
\begin{align*}
R_{SL}(\theta_0)&\leq \frac{1}{\beta} I(\Theta=\theta_0; Y'| X^n, Y^n,X'),
\end{align*}
where the distribution  of $(\Theta, X', Y', X^n, Y^n)$ is given by $q(\theta)p_{\theta}(x',y')\prod_{i=1}^n p_{\theta}(x_i, y_i)$ for any choice of $q(\theta)$. It also holds that 
\begin{align*}
R_{SSL}(\theta_0)&\leq \frac{1}{\beta} I(\Theta=\theta_0; Y'| X^n, Y^n, \td X^m,X'),
\end{align*}
where  the distribution  of $(\Theta, X', Y', X^n, Y^n,\td X^m)$ is given by $q(\theta)p_{\theta}(x',y')\prod_{i=1}^n p_{\theta}(x_i, y_i)\prod_{j=1}^m p_{\theta}(\td x_j)$ for any choice of $q(\theta)$.
\label{lemma:expconcave}
\end{lemma}

\begin{remark}
Instead of using the classical empirical risk minimization (ERM) approach to generate the hypothesis, we use a Bayes method for the prediction (cf. Equation (\ref{eq:hat_w}) in the proof).  In the context of universal prediction, this method is shown to produce an optimal universal predictor under appropriate conditions, in the sense that the average excess risk vanishes as the number of samples increases~(\cite{merhav_universal_1998}, \cite{clarke_barron_1990}).
\end{remark}

\begin{proof}
Let $Z^n$ denote $n$ pairs of i.i.d. data representing the training data and $Z'=(X', Y')$ another i.i.d. pair representing the test data.   Recall that 
\begin{align*}
R_{SL}(\theta_0)&= \min_w\EE{\theta_0}{\ell(w,Z')-\ell(w^*, Z')}.
\end{align*}
To obtain an upper bound to the above quantity, for each $x'$, we choose the hypothesis $w_{Z^n}$ to be 
\begin{align}
\hat w_{Z^n}(x'):=\text{argmin}_w \EE{Q}{\ell(w(x'),Y')|X^{n}, Y^{n}, X'=x'},
\label{eq:hat_w}
\end{align}
where the distribution $Q$ over $\mathcal X^{n+1}\times \mathcal Y^{n+1}$ is chosen to be
\begin{align}
Q(x^{n+1},y^{n+1}):=\int \prod_{i=1}^{n+1}p_{\theta}(x_i,y_i)q(\theta)d\theta
\label{eq:Q_sl}
\end{align}
for some $q(\theta)$ that  we can choose  to suit our needs. More precisely,  the term $\EE{Q}{\ell(w(x'),Y')|X^{n}, Y^{n}, X'=x'}$ is given by
\begin{align*}
&\EE{Q}{\ell(w(x'),Y')|X^{n}, Y^{n}, X'=x'}\\
&=\sum_{y'} Q(y'|X^n, Y^n,x')\ell(w(x'),y'),
\end{align*}
where the conditional distribution  $ Q(y'|x^n, y^n,x')$ is induced from $Q(x^{n+1}, y^{n+1})$ defined in (\ref{eq:Q_sl}). Notice that $\hat w_{z^n}$ does not depend on $\theta_0$. With this choice,  we have
\begin{align*}
&R_{SL}(\theta_0)\leq \EE{\theta_0}{\ell(\hat w_{Z^n},Z')-\ell(w^*, Z')}=\\
&\frac{1}{\beta}\sum_{z^n,x'}p_{\theta_0}(z^n,x')\sum_{y'}p_{\theta_0}(y'|z^n,x')(\beta\ell(\hat w_{Z^n},z')-\beta\ell(w^*,z')).
\end{align*}
Now we upper bound the term
\begin{align*}
&\sum_{y'}p_{\theta_0}(y'|z^n,x')(\beta\ell(\hat w_{Z^n},z')-\beta\ell(w^*,z'))\\
&=\sum_{y'}p_{\theta_0}(y'|z^n,x')(\log\frac{e^{-\beta\ell(w^*(x'),y')}Q(y'|z^n,x')}{e^{-\beta\ell(\hat w_{Z^n}(x'),y')}p_{\theta_0}(y'|z^n, x')}\\
&+ \log\frac{p_{\theta_0}(y'|z^n,x')}{Q(y'|z^n,x')})\\
&\leq \log \sum_{y'}Q(y'|z^n,x')\frac{e^{-\beta\ell(w^*(x'),y')}}{e^{-\beta\ell(\hat w_{Z^n}(x'),y')}}\\
&+D(p_{\theta_0}(Y'|z^n,x')||Q(Y'|z^n,x'))\\
&\leq D(p_{\theta_0}(Y'|z^n,x')||Q(Y'|z^n,x')).
\end{align*}
The last inequality holds because 
\begin{align*}
 \sum_{y'}Q(y'|z^n,x')\frac{e^{-\beta\ell(w^*(x'),y')}}{e^{-\beta\ell(\hat w_{Z^n}(x'),y')}}\leq 1.
\end{align*}
Indeed, as $\hat w_{Z^n}$ is chosen to be the minimizer of the expected value of $\ell(w(x'),y')$ under the distribution of $Q(y'|z^n,x')$,  Lemma \ref{lemma:optimality} (stated at the end of this section) shows that this expectation is smaller or equal to $1$. Consequently,
\begin{align*}
&\E{\ell(\hat w_{Z^n},Z')-\ell(w^*, Z')}\\
&\leq \frac{1}{\beta}\sum_{z^n,x'}p_{\theta_0}(z^n,x') D(p_{\theta_0}(Y'|z^n,x')||Q(Y'|z^n,x'))\\
&=\frac{1}{\beta}D(p_{\theta_0}(Y'|Z^n,X')||Q(Y'|Z^n,X')|p_{\theta_0}(Z^n,X'))\\
&=\frac{1}{\beta} I(Y';\Theta=\theta_0|Z^n,X'),
\end{align*}
%
where the last equality holds because the choice of the distribution $Q$ in (\ref{eq:Q_sl}). To see this, recall the representation in (\ref{eq:mutualinfo_cond}).  
In this expression, replace $X$ with $\theta$, $Y$ with $Y'$, and $Z$ with $Z^n, X'$ for our argument. It can be easily verified that due to the choice of $Q$ in (\ref{eq:Q_sl}), we have the claimed result.


The derivation  of the upper bound on  $R_{SSL}$ is similar to the above derivations, and we only highlight the difference. Similarly,  the hypothesis$\hat w_{Z^n,\td X^m}$ in the SSL case is chosen to be
\begin{align*}
&\hat w_{Z^n,\td X^m}(x')\\
&:=\text{argmin}_w \EE{Q}{\ell(w(x'),Y')|X^{n}, Y^{n},\td X^m,X'=x'}
\end{align*}
where the distribution $Q$ over $\mathcal X^{n+1}\times \mathcal Y^{n+1}\times\mathcal X^m$ is chosen to be
\begin{align*}
Q(x^{n+1},y^{n+1},\td x^m):=\int \prod_{j=1}^mp_{\theta}(\td x_j)\prod_{i=1}^{n+1}p_{\theta}(x_i,y_i)q(\theta)d\theta.
\end{align*}
We have
\begin{align*}
&R_{SSL}(\theta_0)\leq \EE{\theta_0}{\ell(\hat w_{Z^n,\td X^m},Z')-\ell(w^*, Z')}\\
&=\frac{1}{\beta_0}\sum_{z^n, \td x^m,x'}p_{\theta_0}(z^n, \td x^m,x')\sum_{y'}p_{\theta}(y'|z^n,\td x^m,x')\\
&\cdot(\beta\ell(\hat w_{Z^n,\td X^m},z')-\beta\ell(w^*,z'))\\
&\leq \frac{1}{\beta}D(p_{\theta_0}(Y'|Z^n,\td X^m,X')||Q(Y'|Z^n, \td X^m)|p_{\theta_0}(Z^n,\td X^m,X')).
\end{align*}
The last inequality holds because we can show 
\begin{align*}
&\sum_{y'}p_{\theta_0}(y'|z^n,\td x^m,x')(\beta\ell(\hat w_{Z^n,\td X^m},z')-\beta\ell(w^*,z'))\\
&\leq D(p_{\theta_0}(Y'|z^n,\td x^m,x')||Q(Y'| z^n, \td x^m,x'))
\end{align*}
in the same way as in the proof of $R_{SL}$ using Lemma \ref{lemma:optimality}. This concludes the proof.
\end{proof}

There is  one important class of loss function which is not covered in the above lemma, namely the $0-1$  function. This loss function is mostly used in classification problems where the alphabet $\mathcal Y$ is a finite set,  defined as
\begin{align}
\ell(w(x),y)=\begin{cases}
0 \quad \text{if }w(x)=y\\
1 \quad \text{otherwise}
\end{cases}.
\label{eq:01loss}
\end{align}

We will establish a similar result for the $0-1$ loss under the following assumption.
\begin{condition}(Massart noise condition)
Given the density function $p_{\theta}(x, y)$, we assume that for all $\theta\in \Lambda$ and all $x\in\mathcal X$ there exists  some $y$ (depending on $x$) and $1<a<\infty$ so that
\begin{align}
\frac{p_{\theta}(y|x)}{p_{\theta}(y'|x)}\geq a,
\label{eq:ratio_alpha}
\end{align}
for any $y'\neq y$. 
\label{condition_massart}
\end{condition}
The above condition is often called \textit{Massart noise condition} (see  \cite{massart_risk_2006}), which is a specialization of the condition proposed in \cite{tsybakov_optimal_2004} and \cite{mammen_smooth_1999}. Intuitively, it means that given any feature vector $x$, it is ``easy" to determine the label with the highest likelihood, by requiring that the ratio of any two likelihood  functions is  at least some value strictly larger than $1$.

The following lemma shows that the same result holds with this additional assumptions on the distribution $p_{\theta}(y|x)$ .
\begin{lemma}[Upper bound on  risk for $0-1$ loss]
Assume that $\ell(w(x),y)$ is the $0-1$ loss function defined in (\ref{eq:01loss})  where $\mathcal Y$ is a finite set.  Also assume that $p_{\theta}(y|x)$ satisfies Condition \ref{condition_massart} with some $a>1$. Then the bounds on $R_{SL}$ and $R_{SSL}$ in Lemma \ref{lemma:expconcave} hold with $\beta=\log a$.
\label{lemma:01}
\end{lemma}
\begin{proof}
The proof proceeds in the same way as in the proof of Lemma \ref{lemma:expconcave}. In particular, we choose the hypothesis $\hat w_{z^n}$ as in (\ref{eq:hat_w}). With this choice,  we can upper bound $R_{SL}$ as
\begin{align*}
&R_{SL}(\theta_0)\leq \EE{\theta_0}{\ell(\hat w_{Z^n},Z')-\ell(w^*, Z')}\\
&=\sum_{z^n,x'}p_{\theta_0}(z^n,x')\sum_{y'}p_{\theta_0}(y'|z^n,x')(\ell(\hat w,z')-\ell(w^*,z')).
\end{align*}
In the following bound, we take logarithm with the base $a$ for some $a>1$ to be determined later.  We use  $D_a$  and $I_{a}$ to denote the KL divergence and mutual information where the logarithm is with base $a>1$. It holds that
\begin{align*}
&\sum_{y'}p_{\theta_0}(y'|z^n,x')(\ell(\hat w_{Z^n},z')-\ell(w^*,z'))\\
&=\sum_{y'}p_{\theta_0}(y'|z^n,x')\Bigg(\log_{a}\frac{a^{-\ell(w^*,z')}Q(y'|z^n,x')}{a^{-\ell(\hat w_{Z^n},z')}p_{\theta_0}(y'|z^n, x')}\\
&+ \log_{a}\frac{p_{\theta_0}(y'|z^n,x')}{Q(y'|z^n,x')}\Bigg)\\
&\leq \log_{a} \sum_{y'}Q(y'|z^n,x')\frac{a^{-\ell(w^*(x'),y')}}{a^{-\ell(\hat w_{Z^n}(x'),y')}}\\
&+D_{a}(p_{\theta_0}(Y'|z^n,x')||Q(Y'|z^n,x'))\\
&\leq D_{a}(p_{\theta_0}(Y'|z^n,x')||Q(Y'|z^n,x')),
\end{align*}
if we can show that
\begin{align*}
 \sum_{y'}Q(y'|z^n,x')\frac{a^{-\ell(w^*(x'),y')}}{a^{-\ell(\hat w_{Z^n}(x'),y')}}\leq 1.
\end{align*}
Notice that  Lemma \ref{lemma:optimality} does not apply here when  $\ell(w(x),y)=\ve 1_{w(x)\neq y}$. To show the above inequality, we use $\hat y$ to denote $\hat w_{Z^n}(x')$ and $y^*$ to denote $w^*(x')$ because both $\hat w$ and $w^*$ belongs to $\mathcal Y$. We can rewrite the LSH of the above inequality as
\begin{align*}
\sum_{y'\neq \hat y, y'\neq y^*} Q(y'|z^n,x')+Q(y^*|z^n,x')a+Q(\hat y|z^n,x')a^{-1}
\end{align*}
because $\ell(y^*,y')=\ell(\hat y,y')=1$ if $y'\neq y^*$ and $y'\neq \hat y$. Then the desired inequality is satisfied if it holds that
\begin{align*}
\sum_{y'\neq \hat y, y'\neq y^*} Q(y'|z^n,x')&+Q(y^*|z^n,x')a\\
&+Q(\hat y|z^n,x')a^{-1}\leq 1,
\end{align*}
or equivalently
\begin{align}
\frac{Q(\hat y|z^n,x')}{Q(y^*|z^n,x')}\geq a.
\label{eq:MAP_ratio}
\end{align}
We show in Supplementary Materials  Section \ref{append:lowerbound} that (\ref{eq:MAP_ratio}) holds under Condition \ref{condition_massart}.

Consequently,
\begin{align*}
&\E{\ell(\hat w_{Z^n},Z')-\ell(w^*, Z')}\\
&\leq\sum_{z^n,x'}p_{\theta_0}(z^n,x') D_{a}(p_{\theta_0}(Y'|z^n,x')||Q(Y'|z^n,x'))\\
&=D_{a}(p_{\theta_0}(Y'|Z^n,X')||Q(Y'|Z^n,X')|p_{\theta_0}(Z^n,X'))\\
&= I_{a}(Y';\Theta=\theta_0|Z^n,X')\\
&=I(Y';\Theta=\theta_0|Z^n,X')/\log a,
\end{align*}
where in the last step we use the change of base again. Setting $\beta=\log a>0$ gives the claimed result for $R_{SL}$. The proof of $R_{SSL}$ follows an almost identical argument.
\end{proof}

The following result was used in the proof of Lemma \ref{lemma:expconcave}. 
\begin{lemma}
Let $w^*$ be the minimizer of $\mathbb E_{Q}\{\ell(w, Z)\}$ where the expectation over $Z$ is taken with respect to the distribution $Q$. Let $w'$ be any other choice of the hypothesis. If $\ell(w,z)$ is $\beta$-exp-concave for all $z$ with $w>0$  then it holds that
\begin{align*}
\mathbb E_Q\left\{\frac{g(w',Z)}{g(w^*,Z)}\right\}\leq 1,
\end{align*}
where $g(w,z):=\exp \{-\beta\ell(w,z)\}$.
\label{lemma:optimality}
\end{lemma}
\begin{proof}
Let $w_{\lambda}=(1-\lambda)w^*+\lambda w'$ be a deviation from the minimizer $w^*$ to another predictor $w'$ characterized by $\lambda$. By the optimality condition of $w^*$, we have
\begin{align*}
\frac{d\mathbb E_{Q}\{-\beta\ell(w_{\lambda}, X)\}}{d\lambda}\Big\rvert_{\lambda=0^+}\leq 0
\end{align*}
Notice that
\begin{align*}
&\frac{d\mathbb E_{Q}\{-\beta\ell(w_{\lambda}, X)\}}{d\lambda}\Big\rvert_{\lambda=0^+}\\&=\lim_{\lambda\rightarrow 0} \frac{1}{\lambda}\E{-\beta\ell(w_{\lambda},X)+\beta\ell(w^*,X)}\\
&=\lim_{\lambda\rightarrow 0} \frac{1}{\lambda}\E{\log \frac{g(w_{\lambda},X)}{g(w^*,X)}}\\
&\geq \lim_{\lambda\rightarrow 0} \frac{1}{\lambda}\E{\log \frac{(1-\lambda)g(w^*,X)+\lambda g(w',X)}{g(w^*,X)}}\\
&=\E{\frac{g(w',X)}{g(w^*,X)}}-1
\end{align*}
where the inequality holds because $g$ is concave in $w$.
\end{proof}

\section{THE CASE WHEN UNLABELED DATA IS EQUALLY USEFUL}
\label{sec:main}

In this section, we  evaluate the mutual information terms in  Lemma \ref{lemma:expconcave} to derive the asymptotic expression for excess risks with additional assumptions on the distributions $p_{\theta}$.


Define Fisher information matrices as
\begin{align*}
I_{XY}(\theta)&:=\E{\frac{\partial}{\partial\theta_j}  \log p(X, Y|\theta)\frac{\partial}{\partial\theta_k}  \log p(X, Y|\theta)}_{j,k}\\
I_{X}(\theta)&:=\E{\frac{\partial}{\partial\theta_j}  \log p(X|\theta)\frac{\partial}{\partial\theta_k}  \log p(X|\theta)}_{j,k}
\end{align*}
for $j,k=1,\ldots, d$. 

The main condition we need on the distribution is:

\textit{Condition 1:} Let $\theta_0$ denote the true parameter. The density $p_{\theta}(x)$ and $p_{\theta}(x,y)$ are twice continuously differentiable at $\theta_0$. The Fisher information matrices  $I_{XY}(\theta_0)$ and $I_X(\theta_0)$ are positive definite, and it holds that $I_{XY}(\theta_0)\succ I_X(\theta_0)$ with respect to the positive definite ordering\footnote{Notice  it always holds that $I_{XY}(\theta_0)\succeq I_X(\theta_0)$ by the chain rule of information matrix.}. 

We also need the following technical conditions.

\textit{Condition 2:} Assume that for all $\theta$ in some neighbourhood of $\theta_0$,  the (normalized) Renyi divergences of order $1+\lambda$
\begin{align*}
&\log \int p(x|\theta_0)^{1+\lambda}p(x|\theta)^{-\lambda}dx\\
&\log  \int p(x,y|\theta_0)^{1+\lambda}p(x,y|\theta)^{-\lambda}dxdy
\end{align*}
are bounded for some small enough $\lambda>0$.

\textit{Condition 3:} Assume that for all $\theta$ in some neighbourhood of $\theta_0$,  the moment generating function
\begin{align*}
\E{e^{\lambda\frac{\partial^2}{\partial \theta_j\partial\theta_k}\log p(X|\theta)}}, \E{e^{\lambda\frac{\partial^2}{\partial \theta_j\partial\theta_k}\log p(X,Y|\theta)}}
\end{align*}
exist for all $j,k=1,\ldots, d$ with some small $\lambda>0$.

\textit{Condition 4:} Let $l:=\nabla \log p(X,Y|\theta_0)$, $\td l:=\nabla \log p(X|\theta_0)$, and $l', \td l'$ an independent copy of $l$ and $\td l$, respectively. The moment generating functions
\begin{align*}
&\E{e^{\lambda l^T(I_{XY}(\theta_0)+I_X(\theta_0))l}}, \E{e^{\lambda l^T(I_{XY}(\theta_0)+I_X(\theta_0))l'}}\\
&\E{e^{\lambda \td l^T(I_{XY}(\theta_0)+I_X(\theta_0))\td l}}, \E{e^{\lambda \td l^T(I_{XY}(\theta_0)+I_X(\theta_0))\td l'}}
\end{align*}
exist for some small enough $\lambda>0$.

A few words are in order with regard to the above conditions. Condition 1 is crucial for our results. Notice that the density functions need to be twice continuously differentiable so we are only dealing with continuous parameters. The positive definiteness of Fisher information matrices is also a key assumption.  In particular, the matrix $I_{X}(\theta_0)$ being positive definite means  that the unlabeled data contains non-trivial information about the whole parameter vector $\theta_0$.   Condition 2, 3, and 4 are technical conditions to ensure that the reminding terms of the approximation to mutual information term in Lemma \ref{lemma:expconcave} decays in a fast enough rate. We point out that though complicated-looking, the existence requirement of divergence and moment generating functions are in general easy to satisfy if Condition 1 holds. Furthermore,  we expect that with a refined analysis, it may be possible to prove the same result without Condition 2, 3, and 4. Indeed, a proof outline is given in \cite{clarke_comment_2012} for a similar result without additional assumptions.

Now we are ready to state the main result of this section.
\begin{theorem}[Learning rate]
Let $n, m$ be the number of labeled and unlabeled data, respectively. Assume that the loss function $\ell(w,z)$ is $\beta$-exponentially concave in $w$ for all $z$.  Assuming that Condition 1, 2, 3, and 4 above hold, we have the following statements.
\begin{itemize} 
\item[1)] (Semi-supervised learning) Let $m=\alpha n$ for some $\alpha>0$. It holds that
\begin{align*}
R_{SSL}(\theta_0)\leq \frac{K_1(\theta_0)}{2n}+o(1/n)
\end{align*}
where 
\begin{align*}
K_1(\theta_0):=&\frac{1}{\beta}\Tr((I_{XY}(\theta_0)+\alpha I_X(\theta_0))^{-1}I_{XY}(\theta_0)))\\
&-\Tr((I_{XY}(\theta_0)+\alpha I_X(\theta_0))^{-1}I_{X}(\theta_0))).
\end{align*}
\item[2)] (Supervised learning) Let $m=0$ .  It holds that 
\begin{align*}
R_{SL}(\theta_0)\leq \frac{K_2(\theta_0)}{2n}+o(1/n)
\end{align*}
where $K_2(\theta_0):=\frac{1}{\beta} (d-\Tr(I_{XY}^{-1}(\theta_0)I_X(\theta_0)))$.

\item[3)] (Semi-supervised learning with many unlabeled data) Let $m= n^{1+\gamma}$ for some $\gamma>0$. It holds that
\begin{align*}
R_{SLL}(\theta_0)\leq \frac{K_3(\theta_0)}{2n^{1+\gamma}}+o(1/n^{1+\gamma})
\end{align*}
where $K_3(\theta_0):=\frac{1}{\beta} (\Tr(I_X^{-1}(\theta_0)I_{XY}(\theta_0))-d)$.
\end{itemize}
If the loss function $\ell(w,z)$ is the $0-1$ function satisfying Condition \ref{condition_massart} with parameter $a$, the above bounds hold with $\beta=\log a$.
\label{thm:main}
\end{theorem}

\begin{remark}
It can be checked straightforwardly that we have $0<K_1(\theta_0)\leq K_2(\theta_0)\leq K_3(\theta_0)$.  Item $1)$ and $2)$ show that if the number of unlabeled data $m$ grows as $O(n)$ with $n$ being the number of labeled data, then the learning rate for both supervised and semi-supervised learning converges  as $O(1/n)$ where additional unlabeled data only improves the constant from $K_2$ to $K_1$.  

Item $3)$  shows that when the number of unlabeled data $m$ is dominating $n$,  then the convergence rate is $O(1/m)$ with a larger  constant $K_3$. Hence the learning rate can be improved from $O(1/n)$ to $O(1/n^{1+\gamma})$ if the number of unlabeled data grows superlinearly with respect to $n$.  In other words, unlabeled data is equally useful in terms of the convergence rate in this case, and the loss due to not having all data labeled is only shown in the constant  ($K_3\geq K_2$). We point out that a similar observation has  also been made in \cite{gopfert_when_2019}.


\end{remark}

\begin{remark}
It is interesting to exam the upper bound in Item 1) (semi-supervised learning)  for $d=1$ when both $I_{XY}$ and $I_X$ are scalars.  In this case, the upper bound takes the form
\begin{align*}
R_{SLL}(\theta_0)\leq O\left(\frac{1}{nI_{XY}(\theta_0)-mI_X(\theta_0)}\right).
\end{align*}
In other words, one labeled data is $I_{XY}(\theta_0)/I_X(\theta_0)$  more valuable than unlabeled data as far as the convergence rate is concerned (notice that $I_{XY}(\theta_0)/I_X(\theta_0)\geq 1$) for the regime $m\sim n$.  The same result was obtained in \cite{castelli_relative_1996} for the simple mixture model with $d=1$. Our theorem extend this result to more general cases when it is not necessarily a mixture model.
\label{remark:cover}
\end{remark}

The proof of Theorem \ref{thm:main}  relies on the following asymptotic characterization of the KL divergence between $p(Z^n,\td X^m|\theta)$ and the ``mixture" distribution  $Q(\cdot)$  defined as $Q(\cdot):=\int q(\theta)p(\cdot|\theta)d\theta$.
\begin{lemma}[Asymptotic expression of KL-divergence]
Assume that Condition 1, 2, 3, and 4 hold.  Let both $m, n$ increase in a way that either $m=\alpha n$ for some $\alpha>0$, or $m=n^{1+\gamma}$ for some $\gamma>0$. Then there exists a  prior $q(\theta)$ so that
\begin{align*}
&D(p(Y^n, X^n, \td X^{m}|\theta)||Q(Y^n, X^n, \td X^{m}))\\
&=\frac{d}{2}\log\frac{1}{2\pi e}+\log\frac{1}{q(\theta)}\\
&+\frac{1}{2}\log |nI_{XY}(\theta)+mI_X(\theta)|+ o(1/\max\{n,m\})
\end{align*}
Under the same assumptions and let $m=0$, we have
\begin{align*}
&D(p(Y^n, X^n|\theta)||Q(Y^n, X^n))\\
&=\frac{d}{2}\log\frac{1}{2\pi e}+\log\frac{1}{q(\theta)}+\frac{1}{2}\log |nI_{XY}(\theta)|+o(1/n).
\end{align*}
\label{lemma:KL}
\end{lemma}

The same approximation result has been established in \cite{clarke_barron_1990} where the authors showed that the reminder term vanishes as $n\rightarrow\infty$. In our case, we need to show that it vanishes with a fast enough rate $o(1/n)$. This lemma is proved in Section \ref{append:main} in  the Supplementary Materials.

Equipped with Lemma \ref{lemma:KL},  we are ready to give a proof of Theorem \ref{thm:main}.
\begin{proof}
We first show the upper bound on $R_{SL}$.  Using the chain rule of mutual information, we have
\begin{align*}
I(Y';\Theta|Z^n,X')&=I(Y',X',Z^n;\Theta)-I(Z^n,X';\Theta).
\end{align*}
Using Lemma \ref{lemma:KL}, we have
\begin{align*}
&I(\Theta=\theta_0; Z^n, X')= D(p_{\theta_0}(X', X^n,Y^n)||Q(X', X^n,Y^n))\\
&=\frac{d}{2}\log\frac{1}{2\pi e}+ \log \frac{\sqrt{|nI_{XY}(\theta_0)+I_{X}(\theta_0)|}}{q(\theta_0)}+o(1/n)\\
&=\frac{d}{2}\log\frac{1}{2\pi e}+  \log\frac{\sqrt{|nI_{XY}(\theta_0)|}}{q(\theta_0)}\\
&+ \log \frac{\sqrt{|\ve I + \frac{1}{n} I_{XY}^{-1}(\theta_0)I_{X}(\theta_0)|}}{q(\theta_0)}+o(1/n).
\end{align*}
By noticing that $(Y',X',Z^n)$ has the same distribution as $Z^{n+1}$, we have
\begin{align*}
&I(\Theta=\theta_0; Y',X', Z^n)=I(\Theta=\theta_0; Z^{n+1})\\
&= D(p_{\theta_0}(X^{n+1},Y^{n+1})||Q(X^{n+1},Y^{n+1}))\\
&=\frac{d}{2}\log\frac{1}{2\pi e}+\log \frac{\sqrt{|(n+1)I_{XY}(\theta_0)|}}{q(\theta_0)}+o(1/n). 
\end{align*}
Hence
\begin{align*}
&I(\Theta=\theta_0; Y'| X^n, Y^n,X')\\
&=I(\Theta=\theta_0; Y',X',Z^n)-I(\Theta=\theta_0; X',Z^n)\\
&= \frac{d}{2}\log\frac{n+1}{n}- \frac{1}{2}\log |\ve I+\frac{1}{n}I_{XY}^{-1}(\theta_0)I_X(\theta_0)|+o(1/n).
\end{align*}
Using the expansion of determinant:
\begin{align}
\left|\ve I+\frac{1}{n}A\right|=1+\frac{1}{n}\Tr(A)+o(1/n),
\label{eq:det}
\end{align}
we have
\begin{align*}
&I(\Theta=\theta_0; Y'| X^n, Y^n,X')\\
&\leq\frac{d}{2}\log\frac{n+1}{n}-\frac{1}{2}\log\Big(1+\frac{1}{n} \Tr(I_{XY}^{-1}(\theta_0)I_Y(\theta_0))\\
&+o(1/n)\Big)+o(1/n)\\
&= \frac{d}{2n}-\frac{\Tr(I_{XY}^{-1}(\theta_0)I_X(\theta_0))}{2n}+o(1/n)
\end{align*}
where we use the fact $\log(1+\frac{C}{n} )=\frac{C}{n}+o(1/n)$ for some  $C>0$. The bound on $R_{SL}$ follows from Lemma \ref{lemma:expconcave} by defining $K_2$ as in the theorem.

To bound $R_{SSL}$ for the case $m=\alpha n$, we have
\begin{align*}
&I(\Theta=\theta_0; Y'| X^n, Y^n, \td X^m,X')\\
&=I(\Theta=\theta_0; X', Y', X^n, Y^n,\td X^m)\\
&-I(\Theta=\theta_0;X^n, Y^n,\td X^m, X')\\
&=\frac{1}{2}\log |(n+1)I_{XY}(\theta_0)+\alpha nI_X(\theta_0)|\\
&-\frac{1}{2}\log |nI_{XY}(\theta_0)+(\alpha n+1)I_X(\theta_0)|+o(1/n)\\
&=\frac{1}{2}\log |\ve I+\frac{1}{n}A^{-1}I_{XY}(\theta_0)|-\frac{1}{2}\log |\ve I+\frac{1}{n}A^{-1}I_{XY}(\theta_0)|\\
&+o(1/n)
\end{align*}
where $A:=I_{XY}(\theta_0)+\alpha I_X(\theta_0)$. Using (\ref{eq:det}) again together with $\log(1+C/n)=C/n+o(1/n)$, we obtain the claimed constant $K_1(\theta_0)$.

For the case $m=n^{1+\gamma}$, Lemma \ref{lemma:KL} shows that 
\begin{align*}
&I(\Theta=\theta_0; Y'| X^n, Y^n, \td X^m,X')\\
&=I(\Theta=\theta_0; X', Y', X^n, Y^n,\td X^m)\\
&-I(\Theta=\theta_0; X^n, Y^n,\td X^m, X')\\
&=\frac{1}{2}\log |\ve I+\frac{n+1}{n^{1+\gamma}}I_{X}^{-1}(\theta_0)I_{XY}(\theta_0)|\\
&-\frac{1}{2}\log |\ve I+\frac{1}{n^{1+\gamma}}I_{X}^{-1}(\theta_0)(nI_{XY}(\theta_0))+I_X(\theta_0)|\\
&+o(1/n^{1+\gamma}).
\end{align*}
Using (\ref{eq:det}) and $\log(1+C/n)=C/n+o(1/n)$ again, the main terms in the above expression simplifies to
\begin{align*}
&\frac{n+1}{2n^{1+\gamma}}\Tr(I_{X}^{-1}(\theta_0)I_{XY}(\theta_0))-\frac{n\Tr(I_{X}^{-1}(\theta_0)I_{XY}(\theta_0))+d}{2n^{1+\gamma}}\\
&=\frac{\Tr(I_{X}^{-1}(\theta_0)I_{XY}(\theta_0))-d}{2n^{1+\gamma}},
\end{align*}
which gives the desired constant $K_3(\theta_0)$.
\end{proof}


\section{LOWER BOUND}
Theorem \ref{thm:main} only gives upper bounds on the learning rate, and it is natural to ask whether the $O(1/n)$ rate for supervised learning is optimal. In other words, whether the rate improvement in semi-supervised learning  is genuinely due to additional unlabeled data, and not possible with labeled data alone. The next result shows this is indeed the case. Specifically, we show for a certain type of loss function, the mutual information characterization in Lemma \ref{lemma:expconcave} is in fact exact for the worst-case $\theta_0$.

To formulate the result, define the maximin excess risk of the supervised learning algorithm to be
\begin{align*}
R_{SL}^*:=	\max_{\theta_0}R_{SL}(\theta_0) =\max_{\theta}\min_{w }\E{R_{\theta}(w_{Z^n})}.
\end{align*}
We will consider the self-information loss function mentioned in Section \ref{sec:bounds}. Formally, for the case when $|\mathcal Y|$ is finite,  the self-information loss function can be defined as
$\ell(w_{Z^n}(x), y)=-\log w_{Z^n}(x)^T \bar y$
where $w_{Z^n}(x)$ is a length-$|\mathcal Y|$ probability vector (a vector with nonnegative entries and sum to $1$) depending on $Z^n$ and $x$, and $\bar y$ is the ``one-hot"  vector  of length-$|\mathcal Y|$  that consisting $1$ at the entry with value $y$ and $0$ otherwise. Namely $\ell(w_{Z^n}(x), y)$ returns a value in $[0,1]$ denoting the predicted probability that $Y=y$. In the case when $Y$ is continuous, the self-information loss can be written as $-\log w(y)$ where $w$ denotes a distribution on $Y$.

\begin{lemma}[Exact  excess risk for self-information loss]
For the self-information loss function defined above, we have
\begin{align*}
R_{SL}^*=\max_{q(\theta)}I(Y';\Theta|X^n,Y^n,X')
\end{align*}
where the distribution  of $(\Theta, X', Y', X^n, Y^n)$ is given by $q(\theta)p_{\theta}(x',y')\prod_{i=1}^n p_{\theta}(x_i, y_i)$.
\label{lemma:lower_SL}
\end{lemma}

This lemma follows directly from the classical ``redundancy-capacity" result in universal prediction, and a proof can be found in e. g. \cite{Gallager_capacity}. Combined with the result in Lemma \ref{lemma:KL}, we see that the $O(1/n)$ convergence rate is optimal for the self-information loss function in the worst case, and cannot be improved.

\section{COMPARISONS AND EXAMPLES}
\label{sec:compare}
In this section, we compare our results with several existing results  in the literature, and comment on the difference in the problem formulation. 

As  mentioned in Remark \ref{remark:cover},  \cite{castelli_relative_1996} studied the mixture model where the individual density functions are known but the mixing parameter is unknown. It was shown that in this case, unlabeled and labeled data play the same role in terms of convergence rate. Theorem \ref{thm:main} extends this result  to more general cases.  More precisely, \cite{castelli_relative_1996} studied the distribution (formulated with our notation)
\begin{align*}
p(x,y|\theta)=(\theta f_1(x))^{\ve 1_{y=1}}(\bar \theta f_2(x))^{\ve 1_{y=2}}
\end{align*}
where $\bar \theta:=1-\theta$ and $f_1$ and $f_2$ are two density functions.  Take the example in \cite[Sec. III, A]{castelli_relative_1996} where $f_1(x)=2x$ and $f_2(x)=2(1-x)$ for $x\in[0,1]$ and $\theta_0=1/2$. It is straightforward to show that the conditions used in Theorem \ref{thm:main} are satisfied, and  that $I_{XY}(\theta_0)=4$ and $I_X(\theta_0)=4/3$. In the case when $n\sim m$, Theorem \ref{thm:main} states that
\begin{align*}
R_{SSL}\leq O\left(\frac{1}{4n+\frac{4}{3}m}\right)
\end{align*}
which recovers the result in \cite{castelli_relative_1996}. However, it is also shown in this paper that if both the association and the mixing parameters are unknown, then the learning rate, to the first order, decays exponentially fast with the number of labeled data if $e^{n}m^{-1}\rightarrow 0$ (see \cite[Thm. 2]{castelli_relative_1996} for details). This observation cannot be deduced from our results because it is not covered in our problem formulation. Indeed, the assumption with unknown association can only be converted into our model  by introducing  \textit{discrete} parameters, where our current model assumes that the unknown parameter takes continuous values in $\mathbb R^d$.   The same comment also applies to the  problem formulation in \cite{rigollet_generalization_2007}, where the notion of ``cluster" is in spirit similar to a model with discrete unknown parameters, which is not handled in our problem formulation.

\cite{gopfert_when_2019}  studied different problem formulations under which the unlabeled data can improve the learning rate. In particular, the authors discussed three different  approaches  in the literature, namely \textit{``improvements via idealistic SSL"}, \textit{``improvements via sample size dependent classes"}, and \textit{``improvements via easy marginal estimation"} (see the paper for  detailed information). Our problem formulation does not fall under the first two categories, as we neither assume that the true marginal distribution of the label is known to the learning algorithm, nor allow the distribution depend on the number of samples. The reason for the improvement in Theorem \ref{thm:main} is because $X$ contains non-trivial information about the whole parameter vector $\theta_0$ (Condition 1), hence is closer to the third category. Nevertheless, our result  confirms the observation in \cite{gopfert_when_2019} that we can have non-trivial rate change in SSL if $m$ grows faster than $n$ (see, e. g. Example 3 in the paper).    Lastly, we point out that our result does not contradict with the lower bound $O(1/\sqrt{n})$  in \cite[Appendix B]{gopfert_when_2019}, because they allow the distribution to depend on the number of samples.

We provide one additional example to illustrate our result.   Consider a model where $X,Y$ is given by 
\begin{align*}
X=Y+Z
\end{align*}
with $Y\sim\mathcal N(0,\sigma^2)$ and $Z\sim\mathcal N(\mu, 1)$ being independent. The (unknown) parameter $\theta_0$ is a two-dimensional vector $\theta_0=(\mu, \sigma^2)$ with some $\sigma^2>0$ and $\mu\in\mathbb R$.  We would like to predict $Y$ from $X$ given labeled  data $(X_i, Y_i), i=1,\ldots, n$ and possible unlabeled data $\tilde X_i, i=1,\ldots, m$ with the self-information loss function $-\log w(y)$ where $w$ is a distribution on $\mathbb R$. Notice this is not a mixture model so it is not clear from previous results that the unlabeled data is useful. However we expect that additional $\hat X_i$ should be helpful as it does provide information about the variance $\sigma^2$, which is essentially what we need for predicting $Y$ under the self-information loss. Indeed,  our result in Theorem \ref{thm:main} confirms the intuition.  Straightforward calculation shows that the Fisher information matrices are
\begin{align*}
I_X=\begin{pmatrix}
\frac{1}{\sigma^2+1} &0\\
0 &\frac{1}{2(\sigma^2+1)^2}
\end{pmatrix},\quad I_{XY}=\begin{pmatrix}
1 &0\\
0 &\frac{1}{2\sigma^4}
\end{pmatrix}
\end{align*}
which are both positive definite and Condition 1 holds. We can also verify straightforwardly that Condition 2, 3, and 4 hold for this Gaussian model.

It can be calculated that the constants in Theorem \ref{thm:main} are given by
\begin{align*}
K_1&=\frac{\sigma^2}{\alpha+\sigma^2+1}+\frac{1+2\sigma^2}{2\sigma^2+(\alpha+1)\sigma^4+1},\\
K_2&=1+\frac{\sigma^2}{(\sigma^2+1)^2},\\
K_3&=\frac{\sigma^6+2\sigma^2+1}{\sigma^4}.
\end{align*}
We can inspect the SSL case by noting that (in this case $\beta=1$)
\begin{align*}
K_1<\frac{2+4\sigma^2}{1+\sigma^2+\alpha\min\{1,\sigma^4\}}
\end{align*}
hence the semi-supervised learning rate scales as
\begin{align*}
O\left(\frac{1}{(1+\sigma^2)n+\min\{1,\sigma^4\}m}\right).
\end{align*}
In this case, both labeled and unlabeled examples contribute to the learning rate in the same order. Labeled data is $\frac{1+\sigma^2}{\min\{1,\sigma^4\}}$ times more valuable  than unlabeled data in terms of the leading constant.



\subsubsection*{Acknowledgements}

The author would like to thank the anonymous reviewers for their helpful feedback. The author would also like to thank Prof. Bertrand Clarke for valuable discussions.

\bibliographystyle{apalike}
\bibliography{SSL}

\begin{thebibliography}{}

\bibitem[Alirezaei and Mathar, 2018]{alirezaei_exponentially_2018}
Alirezaei, G. and Mathar, R. (2018).
\newblock On {Exponentially} {Concave} {Functions} and {Their} {Impact} in
  {Information} {Theory}.
\newblock In {\em 2018 {Information} {Theory} and {Applications} {Workshop}
  ({ITA})}, pages 1--10, San Diego, CA. IEEE.

\bibitem[Boucheron et~al., 2013]{boucheron_concentration_2013}
Boucheron, S., Lugosi, G., and Massart, P. (2013).
\newblock {\em Concentration {Inequalities}: {A} {Nonasymptotic} {Theory} of
  {Independence}}.
\newblock OUP Oxford.

\bibitem[Castelli and Cover, 1996]{castelli_relative_1996}
Castelli, V. and Cover, T. (1996).
\newblock The relative value of labeled and unlabeled samples in pattern
  recognition with an unknown mixing parameter.
\newblock {\em IEEE Transactions on Information Theory}, 42(6):2102--2117.

\bibitem[Chapelle et~al., 2006]{chapelle_semi-supervised_2006}
Chapelle, O., Schölkopf, B., and Zien, A., editors (2006).
\newblock {\em Semi-supervised learning}.
\newblock Adaptive computation and machine learning. MIT Press, Cambridge,
  Mass.
\newblock OCLC: ocm64898359.

\bibitem[Clarke, 1989]{Bertrand_thesis}
Clarke, B. (1989).
\newblock {\em Asymptotic cumulative risk and Bayes risk under entropy loss,
  with applications}.
\newblock PhD thesis.

\bibitem[Clarke, 1999]{clarke_asymptotic_1999}
Clarke, B. (1999).
\newblock Asymptotic normality of the posterior in relative entropy.
\newblock {\em IEEE Transactions on Information Theory}, 45(1):165--176.

\bibitem[Clarke, 2012]{clarke_comment_2012}
Clarke, B. (2012).
\newblock Comment on {Article} by {Sancetta}.
\newblock {\em Bayesian Analysis}, 7(1):37--44.

\bibitem[Clarke and Barron, 1994]{clarke_jeffreys_1994}
Clarke, B.~S. and Barron, A.~R. (1994).
\newblock Jeffreys' prior is asymptotically least favorable under entropy risk.
\newblock {\em Journal of Statistical Planning and Inference}, 41(1):37--60.

\bibitem[Clarke and Barron, 1990]{clarke_barron_1990}
Clarke, B.~S. and Barron, A, R. (1990).
\newblock Information-theoretic asymptotics of bayes methods.
\newblock {\em IEEE Transactions on Information Theory}, 36(3):453--471.

\bibitem[Gallager, 1974]{Gallager_capacity}
Gallager, R.~G. (1974).
\newblock Source coding with side information and universal coding.
\newblock In {\em unpublished manuscript; also presetned at International
  Symposium on Information Theory (ISIT)}.

\bibitem[G{\"o}pfert et~al., 2019]{gopfert_when_2019}
G{\"o}pfert, C., Ben-David, S., Bousquet, O., Gelly, S., Tolstikhin, I., and
  Urner, R. (2019).
\newblock When can unlabeled data improve the learning rate?
\newblock In {\em Conference on {Learning} {Theory}}, pages 1500--1518.

\bibitem[Mammen and Tsybakov, 1999]{mammen_smooth_1999}
Mammen, E. and Tsybakov, A.~B. (1999).
\newblock Smooth {Discrimination} {Analysis}.
\newblock {\em The Annals of Statistics}, 27(6):1808--1829.
\newblock Publisher: Institute of Mathematical Statistics.

\bibitem[Massart and Nedelec, 2006]{massart_risk_2006}
Massart, P. and Nedelec, E. (2006).
\newblock Risk bounds for statistical learning.
\newblock {\em The Annals of Statistics}, 34(5):2326--2366.

\bibitem[Merhav and Feder, 1998]{merhav_universal_1998}
Merhav, N. and Feder, M. (1998).
\newblock Universal prediction.
\newblock {\em IEEE Transactions on Information Theory}, 44(6):2124--2147.

\bibitem[Rigollet, 2007]{rigollet_generalization_2007}
Rigollet, P. (2007).
\newblock Generalization {Error} {Bounds} in {Semi}-supervised {Classification}
  {Under} the {Cluster} {Assumption}.
\newblock {\em The Journal of Machine Learning Research}, 8:1369--1392.

\bibitem[Singh et~al., 2009]{singh_unlabeled_2009}
Singh, A., Nowak, R., and Zhu, J. (2009).
\newblock Unlabeled data: {Now} it helps, now it doesn't.
\newblock In Koller, D., Schuurmans, D., Bengio, Y., and Bottou, L., editors,
  {\em Advances in {Neural} {Information} {Processing} {Systems} 21}, pages
  1513--1520. Curran Associates, Inc.

\bibitem[Tsybakov, 2004]{tsybakov_optimal_2004}
Tsybakov, A.~B. (2004).
\newblock Optimal aggregation of classifiers in statistical learning.
\newblock {\em The Annals of Statistics}, 32(1):135--166.

\bibitem[Zhu, 2008]{zhu_review}
Zhu, X. (2008).
\newblock Semi-supservised leraning literature survey.
\newblock {\em Technical Report 1530, University of Wisconsin-Madison}.

\end{thebibliography}

\newpage

\newpage

\newpage

\textbf{Supplementary Materials to \textit{``Semi-Supervised Learning: the Case When Unlabeled Data is Equally Useful"}}

\appendix

\section{Derivation in the proof of Lemma \ref{lemma:01}}
\label{append:lowerbound}

Here we show the claim $\frac{Q(\hat y|z^n,x')}{Q(y^*|z^n,x')}\geq a$ under Condition \ref{condition_massart} in the proof of Lemma \ref{lemma:01}. This ratio can be written out explicitly as 
\begin{align}
\frac{Q(\hat y|z^n, x')}{Q(y^*|z^n,x')}=\frac{\int p_{\theta}(\hat y|x')p(\theta|x', z^{n})d\theta}{\int p_{\theta}(y^*| x')p(\theta|x', z^n)d\theta}.
\label{eq:Q_ratio}
\end{align}
By Condition \ref{condition_massart},  given any $x'$ there exist some $y$ such that $p_{\theta}(y|x')\geq a p_{\theta}(y'|x')$ for any $y'\neq y$. Given the above expression, it implies that $\frac{Q(y|z^n,x')}{Q(y'|z^n,x')}\geq a$. On the other hand, it is well known that for the $0-1$ loss, the optimal classifier in (\ref{eq:hat_w})  is given by the maximum \textit{a posteriori} classifier $\hat y=\text{max}_y Q(y|z^n,x')$, hence we have $\frac{Q(\hat y|z^n,x')}{Q(y^*|z^n,x')}\geq a$.

\section{Proof of Lemma \ref{lemma:KL}}
\label{append:main}

Our proof will largely follow the strategy used in~\cite{clarke_barron_1990}, \cite{Bertrand_thesis}.  The main idea is to approximate the density ratio  $p(Z^n,\td X^m|\theta)/Q(Z^n, \td X^m)$ around $\theta$ using Laplace's method, and control the decay rate of the remaining terms. The definitions of various sets in the proof differ slightly from \cite{clarke_barron_1990} to suit our purpose. As our proof is long but follows closely to the above two references, we will  highlight the different parts and  refer to the original proof for repetitive steps.

We use $p(Z^n,\td X^m|\theta)$ to denote the likelihood defined as
\begin{align*}
p(Z^n,\td X^m|\theta):=\prod_{i=1}^np_{\theta}(X_i,Y_i)\prod_{j=1}^mp_{\theta}(\td X_j).
\end{align*}
Define the (unnormalized) score function as
\begin{align*}
l_{XY}(\theta)&:=\nabla \log p(Z^n|\theta)\\
l_{X}(\theta)&:=\nabla \log p(\td X^n|\theta),
\end{align*}
and  the  (unnormalized) empirical information matrix
\begin{align*}
I^*_{XY}(\theta)&:=-[ \partial^2(\log p(Z^n|\theta))/\partial \theta_j\partial \theta_k]_{j,k=1,\ldots, d }\\
I^*_{X}(\theta)&:=-[ \partial^2(\log p(\td X^{m}|\theta))/\partial \theta_j\partial \theta_k]_{j,k=1,\ldots, d }
\end{align*}


Let $\theta_0$  denote the true parameter that generate the data $Z^n, \td X^m$.  Define $N_{\delta}=\{\theta: \|\theta-\theta_0\|\leq \delta\}$.  For convenience, the norm is defined as
\begin{align*}
\|\xi\|^2=\xi^T(I_{XY}(\theta_0)+I_X(\theta_0))\xi.
\end{align*}
For $0<\epsilon<1$ and $\delta>0$, define
\begin{align*}
A(\delta,\epsilon):=\Bigg\{&\int_{N_{\delta}^c}p(Z^n,\td X^m|\theta)q(\theta)d\theta\\
&\leq \epsilon \int_{N_{\delta}}p(Z^n,\td X^m|\theta)q(\theta)d\theta\Bigg\}.
\end{align*}
For convenience, we also define 
\begin{align*}
I_{n,m}:=nI_{XY}(\theta_0)+mI_X(\theta_0)
\end{align*}
and
\begin{align*}
D(\theta_0):=(l_{XY}(\theta_0)+l_X(\theta_0))^TI_{n,m}^{-1}(l_{XY}(\theta_0)+l_X(\theta_0)).
\end{align*}
Notice that
\begin{align*}
\E{D(\theta_0)}&=\mathbb E\{\Tr((I_{n,m}^{-1})(l_{XY}(\theta_0)+l_X(\theta_0))^T\\
&(l_{XY}(\theta_0)+l_X(\theta_0))\}\\
&=\Tr(I_{n,m}^{-1}(nI_{XY}(\theta)+m I_{XY}(\theta)))=d
\end{align*}
Lastly, define
\begin{align*}
B(\delta,\epsilon):=\{&(1-\epsilon)(\theta-\theta_0)^TI_{n,m}(\theta-\theta_0)\\
&\leq (\theta-\theta_0)^T(I_{XY}^*(\theta')+I_X^*(\theta'))(\theta-\theta_0)\\
&\leq(1+\epsilon) (\theta-\theta_0)^TI_{n,m}(\theta-\theta_0)\\
&\text{ for all }\theta, \theta'\in N_{\delta}\}\\
C(\delta):=&\{D(\theta_0)\leq \min\{n,m\}\delta^2\}
\end{align*}and
\begin{align*}
\rho(\delta, \theta_0):=\sup_{\theta\in N_{\delta}}|\log\frac{q(\theta)}{q(\theta_0)}|.
\end{align*}

In the sequel, we assume that both $m, n$ increase in a way that either $m=\alpha n$ for some $\alpha>0$, or $m=n^{1+\gamma}$ for some $\gamma>0$. Following \cite{clarke_barron_1990}, we have following upper and lower bounds on the density ratio.
\begin{lemma}
\label{lemma:upper_lower}
Assume that the Condition 1 is satisfied, and $q(\theta)$ continuous at $\theta_0$. Then on the set $A\cap B$, we have 
\begin{align*}
\frac{Q(Z^n,\td X^m)}{p(Z^n, \td X^m|\theta_0)}&\leq (1+\epsilon)q(\theta_0)e^{\rho(\delta,\theta_0)}(2\pi)^{d/2}\\
&\cdot e^{1/(2(1-\epsilon))D(\theta_0)}|(1-\epsilon)I_{n,m}|^{-1/2}
\end{align*}
On the set $B\cap C$, we have
\begin{align*}
&\frac{Q(Z^n,\td X^m)}{p(Z^n, \td X^m|\theta_0)}\geq q(\theta_0)e^{-\rho(\delta,\theta_0)}(2\pi)^{d/2}e^{1/(2(1+\epsilon))D(\theta_0)}\\
&\cdot (1-2^{d/2}e^{-\epsilon^2\min\{n,m\}\delta^2/8})|(1+\epsilon)I_{n,m}|^{-1/2}
\end{align*}
\end{lemma}
\begin{proof}
The proof of this lemma is very similar to the proof of \cite[Lemma 4.1]{clarke_barron_1990}, except for minor modifications to account for the different definition of the set $B(\delta, \epsilon)$ and $C(\delta)$. The main idea is to use Laplace's method to approximate the integration in $Q(Z^n,\td X^m)$ around the true parameter $\theta_0$.  We omit the details.
\end{proof}

Recall that  $D(p(X^n, Y^n, \td X^m|\theta)||Q(X^n, Y^n,\td X^m))=\E{\log \frac{p(X^n, Y^n, \td X^m|\theta)}{Q(X^n, Y^n,\td X^m)}}$.  Given the above bounds, we now can  define the reminder term $Re$ as follows.
\begin{align*}
&Re:=\log\frac{p(Z^n,\td X^m|\theta_0)}{Q(Z^n,\td X^m)}\\
&-\left(\frac{d}{2}\log\frac{1}{2\pi}+\log\frac{1}{q(\theta_0)}+\frac{1}{2}\log |I_{n,m}|-D(\theta_0)/2\right).
\end{align*}
It is clear that Lemma \ref{lemma:KL} is established if we  show the expectation of $Re$ converges to $0$ with an appropriate rate, which we will do next.

Equipped with Lemma \ref{lemma:upper_lower}, and using the same argument as  in \cite[pp.464]{clarke_barron_1990} (see also \cite{clarke_jeffreys_1994}), we can show the following  upper bound and lower bounds on $\E{Re}$:
\begin{align}
&\E{Re}\nonumber\\
&\geq  -\log(1+\epsilon)-\rho(\delta,\theta_0)-\frac{\epsilon}{2(1-\epsilon)}d+\frac{d}{2}\log(1-\epsilon)\nonumber\\
&+\pp{(A\cap B)^c}(\log\pp{(A\cap B)^c}+\frac{d}{2}\log\frac{1}{2\pi})\nonumber\\
&-\pp{(A\cap B)^c}\log\frac{\sqrt{|I_{n,m}|}}{q(\theta_0)} \label{eq:Re_lower}
\end{align}
and
\begin{align}
&\E{Re}\leq  \rho(\delta,\theta_0)+\frac{\epsilon}{2(1+\epsilon)} d+\frac{d}{2}\log(1+\epsilon)\nonumber\\
&-\log(1-2^{d/2}e^{-\epsilon^2\min\{m,n\}\delta^2/8})+\E{D(\theta_0)\ve 1_{(B\cap C)^c}}\nonumber\\
&+\pp{(B\cap C)^c }\Bigg(\frac{d}{2}\log\frac{1}{2\pi}+|\log \int_{N_{\delta}}q(\theta)d\theta|\\
&+\log\frac{\sqrt{|I_{n,m}|}}{q(\theta_0)}\Bigg)\nonumber\\
&+n\pp{(B\cap C)^c }\E{f(Z)}+m\pp{(B\cap C)^c }\E{f(\td X)}\nonumber\\
&+(n\pp{(B\cap C)^c })^{\frac{1}{2}}{\E{f^2(Z)}}^{\frac{1}{2}}\nonumber\\
&+(m\pp{(B\cap C)^c })^{\frac{1}{2}}{\E{f^2(\td X)}}^{\frac{1}{2}}\label{eq:Re_upper}
\end{align}
where $f(\cdot):= \sup_{\theta',\theta^{''}\in N_{\delta}}(\theta'-\theta_0)^T\nabla \log p(\cdot|\theta'')$

The following  lemmas (Lemma \ref{lemma:Ac},  \ref{lemma:Bc},  \ref{lemma:Cc}) show that the probability that $(Z^n, \td X^m)$ belongs to each of the set $A^c, B^c$ and $C^c$ is smaller than $O(e^{-\min\{m,n\}\rho})$ for some $\rho>0$.  We also show in Lemma \ref{lemma:Ac} and \ref{lemma:Bc} that  we can take $\epsilon=e^{-\max\{m,n\}r}$ for some $r>0$. Moreover, as we can choose the prior distribution $q(\theta)$ to our liking (cf. Lemma \ref{lemma:expconcave}), we will choose $q(\theta)$ to be the uniform distribution over $\Lambda$, so that $\rho(\delta, \theta_0)=0$. So the first four terms in the lower bound $(\ref{eq:Re_lower})$ scales as $O(\epsilon)=O(e^{-\min\{m,n\}})=o(1/\min\{m,n\})$ for large enough $m$ and $n$. Notice that $|I_{n,m}|$ scales as $\log \max\{m,n\}$, so the last two terms in (\ref{eq:Re_lower}) scale as $O(e^{-\min\{m,n\}}\max\{m,n\} )$ which is also $o(1/\min\{m,n\})$ for large $m$ and $n$.

For the upper bound in (\ref{eq:Re_upper}), by choosing $q(\theta)$  to be the uniform distribution, the first four terms scales as $O(e^{-\min\{m,n\}})$ as in the lower bound. Using the same argument as in \cite[pp. 51]{clarke_jeffreys_1994}, $\E{D(\theta_0)\ve 1_{(B\cap C)^c}}$ can be upper bounded using H\"{o}lder's inequality by $O(\pp{(B\cap C)^{c}}^{s/(1+s)})$ for some $s>0$. Furthermore, we can make $|f|$ a very small constant by choosing $\delta$ sufficiently small. So it is easy to see that the rest terms in (\ref{eq:Re_upper}) are of the order $O(e^{-\min\{m,n\}s/(1+s)})+O(\sqrt{e^{-\min\{m,n\}}\max\{m,n\}})$ which also scales as $o(1/\min\{m,n\})$ for large $m$ and $n$. In the following, we conclude the proof by  showing that the probability of the set $A^c, B^c, C^c$ is upper bounded by an exponentially fast decaying term.

\begin{lemma}[Probability of $A^c$]
Assume Condition 2 holds so that for all $\theta\in N_{\delta}$,  the (normalized) Renyi divergence of order $1+\lambda$
\begin{align*}
\int p(x|\theta_0)^{1+\lambda}p(x|\theta)^{-\lambda}dx,  \int p(x,y|\theta_0)^{1+\lambda}p(x,y|\theta)^{-\lambda}dxdy
\end{align*}
are bounded for some $\lambda>0$ small enough. Let $n'=\max\{n,m\}$. Then for $\delta$ sufficiently small, there is an $r>0$ and $\rho>0$ so that
\begin{align*}
\pp{(Z^n, \td X^m)\in A^c(\delta, e^{-n'r})}=O(e^{-\min\{m,n\}\rho})
\end{align*}
\label{lemma:Ac}
\end{lemma}
\begin{proof}
For simplicity, we use $T$ to denote $(Z^n,\td X^m)$ in the proof. For any given $r'>0$, define the event
\begin{align*}
U=\left\{e^{-n'r'}p(T|\theta_0)<\int_{N_{\delta}}q(\theta)p(T|\theta)d\theta\right\}.
\end{align*}
We have
\begin{align}
&\pp{A^c(\delta,e^{-n'r})}\nonumber\\
&=\pp{\int_{N_{\delta}}p(T|\theta)q(\theta)d\theta< e^{n'r} \int_{N_{\delta}^c}p(T|\theta)q(\theta)d\theta}\nonumber\\
&\leq \mathbb P \Bigg\{U\cap\Big\{ \int_{N_{\delta}}p(T|\theta)q(\theta)d\theta\nonumber\\
&< e^{n'r} \int_{N_{\delta}^c}p(T|\theta)q(\theta)d\theta\Big\}\Bigg\}+\pp{U^c}\nonumber\\
&\leq \pp{p(T|\theta_0)<e^{n'(r+r')}\int_{N^c}q(\theta)p(T|\theta)d\theta}\nonumber\\
&+\pp{e^{nr'}\int_{N_{\delta}} p(T|\theta)q(\theta)d\theta<p(T|\theta_0)}
\label{eq:bound_A}
\end{align}
by intersecting with $U$ and $U^c$.

We first study the second term  in (\ref{eq:bound_A}) and show that it converges to zero exponentially. We follow the argument used in \cite{clarke_asymptotic_1999}. Define $Q(T|N_{\delta})=\int_{N_{\delta}} p(X|\theta)q(\theta|N_{\delta})d\theta$ where $q(\theta|N_{\delta})=q(\theta)/(\int_{N_{\delta}} q(\theta)d\theta)$. Define $\td r=r'-\frac{1}{n}\log \int_{N_{\delta}} q(\theta)d\theta$. Applying Jensen's inequality, we can upper bound the second term in (\ref{eq:bound_A}) as
\begin{align*}
&\pp{\log \frac{p(T|\theta_0)}{Q(T|N_{\delta})}>n'\td r}\\
&\leq \pp{\log p(T|\theta_0)-\int_{N_{\delta}} \log p(T|\theta)q(\theta|N_{\delta})d\theta >n'\td r}\\
&= \mathbb P\Bigg\{\int_{N_{\delta}} \log\frac{p(Z^n|\theta_0)}{p(Z^n|\theta_0)}q(\theta|N_{\delta})d\theta\\
& +\int_{N_{\delta}} \log \frac{p(\td X^m|\theta_0)}{p(\td X^m|\theta)}q(\theta|N_{\delta})d\theta>n'\td r \Bigg\}\\
&= \pp{\sum_{i=1}^ng(Z_i)+\sum_{j=1}^{m}g(X_j)>n'\td r}\\
&\leq\pp{\frac{1}{n'}\sum_{i=1}^ng(Z_i)>\td r/2}+\pp{\frac{1}{n'}\sum_{i=1}^mg(X_j)>\td r/2}\\
&\leq \pp{\frac{1}{n}\sum_{i=1}^ng(Z_i)>\td r/2}+\pp{\frac{1}{m}\sum_{i=1}^mg(X_j)>\td r/2}
\end{align*}
where we define
\begin{align*}
g(\cdot):=\int_{N_{\delta}}\log \frac{p(\cdot|\theta_0)}{p(\cdot|\theta)}q(\theta|N_{\delta})d\theta.
\end{align*}
Notice that the expectation of $g$ is $\int_{N_{\delta}} D(p_{\theta_0}||p_{\theta}) w(\theta|N_{\delta})d\theta$ is less than any fixed $\td r/2$ for $\delta$ sufficiently small. If it holds that for any $\theta$ in $N_{\delta}$,  moment generating functions $\int p(x|\theta_0)e^{\lambda g(x)}dx$  and $\int p(x,y|\theta_0)e^{\lambda g(x,y)}dxdy $ exist for some $\lambda\in I$ where $I$ is an interval including $0$, then using the standard Cram\'er-Chernoff method (see, e. g. \cite{boucheron_concentration_2013}), both probabilities in the last inequality  are upper bounded by terms in the order of $O(e^{-\rho n})$ and $O(e^{-\rho m})$ for some $\rho>0$, respectively.

It can be shown that the existence of the moment generating function is guaranteed if Condition 2 holds.  Indeed, applying Jensen's inequality gives
\begin{align*}
e^{\lambda g(x)}\leq \int \left(\frac{p(x|\theta_0)}{p(x|\theta)} \right)^{\lambda}q(\theta|N_{\delta})d\theta.
\end{align*}
Hence the moment generating function is bounded by
\begin{align*}
\int p(x|\theta_0)\left(\frac{p(x|\theta_0)}{p(x|\theta)} \right)^{\lambda}q(\theta|N_{\delta})d\theta dx
\end{align*}
which is upper bounded by  the (unnormalized) Renyi divergence.

The first term in (\ref{eq:bound_A}) can also be shown to be of the order of $O(e^{-\min\{n,m\}r''})$ for some $r''>0$. The proof  is essentially the same as in \cite[Prop. 6.3]{clarke_barron_1990} (see also \cite[pp. 49-50]{clarke_jeffreys_1994}), and is omitted here.
\end{proof}

\begin{lemma}[Probability of $B^c$]
Assume that  Condition 3 holds. Then for $\delta$ sufficiently small, there is a $\rho>0$ such that
\begin{align*}
\pp{(Z^n, \td X^m)\in B^c(\delta,\epsilon)}=O(e^{-\min\{m,n\}\rho})
\end{align*}
\label{lemma:Bc}
\end{lemma}
\begin{proof}
Using the same argument as in \cite[pp. 42]{Bertrand_thesis}, the set $B(\delta,\epsilon)$ can be rewritten as 
\begin{align*}
\left\{\Big\lvert \frac{\xi^TI_{m,n}^{-1/2}(I^*_{XY}(\theta')+I_{X}^*(\theta')-I_{n,m})I_{m,n}^{-1/2}\xi}{\xi^T\xi}\Big\rvert <\epsilon \right\},
\end{align*}
where $\xi=I_{m,n}^{1/2}(\theta-\theta_0)$, and we can upper bound the probability of $B^c$ by
\begin{align*}
&\pp{(Z^n, \td X^m)\in B^c(\delta,\epsilon)}\\
&\leq \sum_{j,k}\Bigg(\pp{\sup_{|\theta_0-\theta|<\delta}|\frac{1}{n}\sum_{i=1}^nI_{j,k}^*(\theta,i)-\frac{1}{n}\sum_{i=1}^nI_{j,k}^*(\theta_0,i)|>\frac{\epsilon}{4d}}\\
&+\pp{|\frac{1}{n}\sum_{i=1}^nI_{j,k}^*(\theta_0,i)-I_{j,k}(\theta_0,i)|>\frac{\epsilon}{4d}}\\
&+\pp{\sup_{|\theta_0-\theta|<\delta}|\frac{1}{m}\sum_{\ell=1}^m\td I_{j,k}^*(\theta)-\frac{1}{m}\sum_{\ell=1}^m\td I_{j,k}^*(\theta_0)|>\frac{\epsilon}{4d}}\\
&+\pp{|\frac{1}{m}\sum_{\ell=1}^m\td I_{j,k}^*(\theta_0)-\td I_{j,k}(\theta_0)|>\frac{\epsilon}{4d}}\Bigg)
\end{align*}
where we use $I_{j,k}^*(\theta,i), \td I_{j,k}^*(\theta,\ell)$ to denote $-\frac{\partial^2}{\partial \theta_j\partial\theta_k}\log p(Z_i|\theta)$ and $-\frac{\partial^2}{\partial \theta_j\partial\theta_k}\log p(\td X_\ell|\theta)$ respectively, and use $I_{j,k}(\theta), \td I_{j,k}(\theta) $ to denote the  $j,k$ entry  of $I_{XY}(\theta)$ and $I_{X}(\theta)$, respectively. Using the standard  Cram\'er-Chernoff method to replace the Chebyshev inequality with Chernoff inequality (applicable because Condition 3 holds) for the steps in \cite[pp. 43]{Bertrand_thesis}, it is easy to show that the first two terms are upper bounded by $O(e^{-n\rho})$ and the last two terms are upper bounded by $O(e^{-m\rho})$ for some $\rho>0$.
\end{proof}

\begin{lemma}[Probability of $C^c$]
Assume that Condition 4 holds.  Then for  some $\rho>0$, we have
\begin{align*}
\pp{(Z^n,\td X^m)\in C^c(\delta)}\leq O(e^{-\min\{m,n\}\rho})
\end{align*}
\label{lemma:Cc}
\end{lemma}
\begin{proof}
Define $l_i:=\nabla \log p(Z_i|\theta)$ and $\td l_j=\nabla \log p(\td X_j|\theta)$.  We rewrite $D(\theta_0)$ as
\begin{align*}
&D(\theta_0)=(\sum_{i=1}^n l_{i}+\sum_{j=1}^n \td l_j)^TI_{m,n}^{-1}(\sum_{i=1}^n l_{i}+\sum_{j=1}^n \td l_j)\\
&=\sum_{i=1}^nl_i^TI_{m,n}^{-1}l_i+\sum_{k\neq i}l_i^TI_{m,n}^{-1}l_k\\
&+\sum_{j=1}^n\td l_j^TI_{m,n}^{-1}\td l_j+\sum_{k\neq j}\td l_j^TI_{m,n}^{-1}\td l_k
\end{align*}
Then
\begin{align}
&\pp{(Z^n,\td X^m)\in C^c(\delta)}=\pp{D(\theta_0)>\min\{m,n\}\delta^2} \nonumber\\
&\leq \pp{\frac{1}{n}\sum_{i=1}^nl_i^TI_{m,n}^{-1}l_i>\frac{\min\{m,n\}\delta^2}{4n}}\nonumber\\
&+ \pp{\frac{1}{n(n-1)}\sum_{k\neq i}^nl_i^TI_{m,n}^{-1}l_k>\frac{\min\{m,n\}\delta^2}{4n(n-1)}}\nonumber\\
&+\pp{\frac{1}{m}\sum_{j=1}^n\td l_j^TI_{m,n}^{-1}\td l_j>\frac{\min\{m,n\}\delta^2}{4m}}\nonumber\\
&+ \pp{\frac{1}{m(m-1)}\sum_{k\neq j}^n\td l_j^TI_{m,n}^{-1}\td l_k>\frac{\min\{m,n\}\delta^2}{4m(m-1)}}\label{eq:bounds_C}
\end{align}
We can show that each of the four terms  has an exponentially fast decay. To see this notice that
\begin{align*}
&\E{l_i^TI_{m,n}^{-1}l_i}=\Tr(I_{m,n}^{-1}\E{l_i^Tl_i})\\
&\leq \frac{1}{\min\{m,n\}}\Tr((I_{XY}(\theta)+I_X(\theta))^{-1}I_{XY})\\
&\leq \frac{1}{\min\{m,n\}}\Tr((I_{XY}(\theta)+I_X(\theta))^{-1}(I_{XY}(\theta)+I_X(\theta)))\\
&=\frac{d}{\min\{m,n\}}
\end{align*}
where the inequalities hold because $I_{XY}(\theta)$ and $I_X(\theta)$ are positive definite.
\begin{align*}
\E{l_k^TI_{m,n}^{-1}l_i}=\Tr(I_{m,n}^{-1}\E{l_k^Tl_i})=0
\end{align*}
as $l_i$ and $l_k$ are independent. Similarly, we also have
\begin{align*}
\E{\td l_j^TI_{m,n}^{-1}\td l_j}\leq \frac{d}{\min\{m,n\}}
\end{align*}
and $\E{\td l_k^TI_{m,n}^{-1}\td l_k}=0$.

Assume Condition 4 holds, the Chernoff bound shows that  the first term in (\ref{eq:bounds_C}) can be upper bounded by a term of the form $O(e^{-n\rho})$ for some $\rho>0$ if it holds that
\begin{align*}
\frac{\min\{m,n\}\delta^2}{4n}>\frac{d}{\min\{m,n\}}
\end{align*}
which always holds for large enough $n$ for the cases $m=\alpha n$ or $m=n^{1+\gamma}$. Similarly, the second term in (\ref{eq:bounds_C}) can be upper bounded by an exponentially fast decaying term if $\frac{\min\{m,n\}\delta^2}{4n(n-1)}>0$, which is always holds for $\delta>0$. The same argument holds for the last two terms in (\ref{eq:bounds_C}), which can be upper bounded by a term of the order $O(e^{-m\rho})$ for some $\rho>0$.
\end{proof}

In the above, we have given the proof of Lemma \ref{lemma:KL} when $m=\alpha n$ for some $\alpha>0$, or $m=n^{1+\gamma}$ for some $\gamma>0$.  The case when $m=0$ follows an almost identical proof except for minor details (in fact this case is even simpler and closer to the proof in \cite{clarke_barron_1990}), and we will not repeat it here.

\end{document}